%% file: Paper.tex
\documentclass{article}
\PassOptionsToPackage{numbers,compress}{natbib}
\usepackage[preprint]{neurips_2020}
\usepackage[utf8]{inputenc}
\usepackage[T1]{fontenc}
\usepackage{url}
\usepackage{booktabs}
\usepackage{amsfonts}
\usepackage{nicefrac}
\usepackage{microtype}

\usepackage{algorithm}
\usepackage{algorithmicx}
\usepackage[noend]{algpseudocode}
\usepackage{amsmath}
\usepackage{amssymb}
\usepackage{amsthm}
\usepackage{bbm}
\usepackage{bm}
\usepackage{color}
\usepackage{dirtytalk}
\usepackage{dsfont}
\usepackage{enumerate}
\usepackage{graphicx}
\usepackage{listings}
\usepackage{mathtools}
\usepackage{subfigure}
\usepackage{times}
\usepackage{wrapfig}
\usepackage{xspace}

\usepackage[usenames,dvipsnames]{xcolor}
\usepackage[bookmarks=false]{hyperref}
\hypersetup{
  pdftex,
  pdffitwindow=true,
  pdfstartview={FitH},
  pdfnewwindow=true,
  colorlinks,
  linktocpage=true,
  linkcolor=Green,
  urlcolor=Green,
  citecolor=Green
}
\usepackage[capitalize,noabbrev]{cleveref}

\usepackage[backgroundcolor=white]{todonotes}

\newcommand{\commentout}[1]{}
\newcommand{\junk}[1]{}

\Crefname{corollary}{Corollary}{Corollaries}
\Crefname{proposition}{Proposition}{Propositions}
\Crefname{theorem}{Theorem}{Theorems}
\Crefname{definition}{Definition}{Definitions}
\Crefname{assumption}{Assumption}{Assumptions}
\Crefname{example}{Example}{Examples}
\Crefname{remark}{Remark}{Remarks}
\Crefname{setting}{Setting}{Settings}
\Crefname{lemma}{Lemma}{Lemmas}

\usepackage{thmtools}

\newcommand{\cN}{\mathcal{N}}
\newcommand{\cP}{\mathcal{P}}

\newcommand{\realset}{\mathbb{R}}

\newcommand{\E}[1]{\mathbb{E} \left[#1\right]}
\newcommand{\condE}[2]{\mathbb{E} \left[#1 \,\middle|\, #2\right]}

\newcommand{\prob}[1]{\mathbb{P} \left(#1\right)}
\newcommand{\condprob}[2]{\mathbb{P} \left(#1 \,\middle|\, #2\right)}

\newcommand{\abs}[1]{\left|#1\right|}
\newcommand{\ceils}[1]{\left\lceil#1\right\rceil}

\newcommand*\dif{\mathop{}\!\mathrm{d}}
\newcommand{\floors}[1]{\left\lfloor#1\right\rfloor}
\newcommand{\I}[1]{\mathds{1} \! \left\{#1\right\}}

\newcommand{\norm}[1]{\|#1\|}

\newcommand{\set}[1]{\left\{#1\right\}}

\newcommand{\T}{^\top}

\DeclareMathOperator*{\argmax}{arg\,max\,}

\mathchardef\mhyphen="2D

\newcommand{\expthree}{\ensuremath{\tt Exp3}\xspace}

\newcommand{\gradband}{\ensuremath{\tt GradBand}\xspace}

\newcommand{\softelim}{\ensuremath{\tt SoftElim}\xspace}

\newcommand{\ts}{\ensuremath{\tt TS}\xspace}
\newcommand{\ucb}{\ensuremath{\tt UCB1}\xspace}
\newcommand{\ucbv}{\ensuremath{\tt UCB\mhyphen V}\xspace}

\definecolor{codegreen}{rgb}{0,0.6,0}
\definecolor{codegray}{rgb}{0.5,0.5,0.5}
\definecolor{codepurple}{rgb}{0.58,0,0.82}
\definecolor{backcolour}{rgb}{0.95,0.95,0.92}

\lstdefinestyle{mystyle}{
  backgroundcolor=\color{backcolour},   commentstyle=\color{codegreen},
  keywordstyle=\color{magenta},
  numberstyle=\tiny\color{codegray},
  stringstyle=\color{codepurple},
  basicstyle=\ttfamily\footnotesize,
  breakatwhitespace=false,         
  breaklines=true,                 
  captionpos=b,                    
  keepspaces=true,                 
  numbers=left,                    
  numbersep=5pt,                  
  showspaces=false,                
  showstringspaces=false,
  showtabs=false,                  
  tabsize=2
}
\title{Differentiable Bandit Exploration}

\author{
  Craig Boutilier \\
  Google Research \\
  \And
  Chih-Wei Hsu \\
  Google Research \\
  \And
  Branislav Kveton \\
  Google Research \\
  \And
  Martin Mladenov \\
  Google Research \\
  \And
  Csaba Szepesv\'{a}ri \\
  DeepMind / University of Alberta \\
  \And
  Manzil Zaheer \\
  Google Research
}

\begin{document}

\maketitle

\begin{abstract}
Exploration policies in Bayesian bandits maximize the average reward over problem instances drawn from some distribution $\cP$. In this work, we \emph{learn} such policies for an unknown distribution $\cP$ using samples from $\cP$. Our approach is a form of meta-learning and exploits properties of $\cP$ without making strong assumptions about its form. To do this, we parameterize our policies in a differentiable way and optimize them by policy gradients, an approach that is general and easy to implement. We derive effective gradient estimators and introduce novel variance reduction techniques. We also analyze and experiment with various bandit policy classes, including neural networks and a novel softmax policy. The latter has regret guarantees and is a natural starting point for our optimization. Our experiments show the versatility of our approach. We also observe that neural network policies can learn implicit biases expressed only through the sampled instances.
\end{abstract}

\input{Introduction}

\input{Setting}

\input{Optimization}

\input{Gradient}

\input{Algorithms}

\input{Experiments}

\input{RelatedWork}

\input{Conclusions}

\bibliographystyle{plainnat}
\bibliography{References}

\clearpage
\onecolumn
\appendix

\input{ConcavityProof}

\clearpage

\input{GradientProof}

\clearpage

\input{Analysis}

\clearpage

\input{Lemmas}

\clearpage

\input{Experiments2}

\clearpage

\input{RNNDetails}

\end{document}

%% file: Introduction.tex

\section{Introduction}
\label{sec:introduction}

A \emph{stochastic bandit} \cite{lai85asymptotically,auer02finitetime,lattimore19bandit} is an online learning problem where a \emph{learning agent} sequentially pulls arms with stochastic rewards. The agent aims to maximize its expected cumulative reward over some horizon. It does not know the mean rewards of the arms \emph{a priori} and learns them by pulling the arms. This induces the well-known \emph{exploration-exploitation trade-off}: \emph{explore}, and learn more about an arm; or \emph{exploit}, and pull the arm with the highest estimated reward. In a clinical trial, the \emph{arm} might be a treatment and its \emph{reward} is the outcome of that treatment for a patient.

Bandit algorithms are typically designed to have \emph{low regret}, worst-case or instance-dependent, for some problem class of interest to the algorithm designer \cite{lattimore19bandit}. While regret guarantees are reassuring, this approach often results in algorithms that are overly conservative, since they do not exploit the full properties of the problem class or objective. We explore an alternative view,
which is to \emph{learn} a bandit algorithm. Specifically, we assume that the agent has access to bandit instances sampled from an unknown distribution $\cP$ and attempts to learn a bandit algorithm that achieves high \emph{Bayes reward}, the average reward over the instances drawn from $\cP$. In essence, we automate the learning of policies for \emph{Bayesian bandits} \cite{berry85bandit}. Our approach can be viewed as a form of \emph{meta-learning} \cite{thrun96explanationbased,thrun98lifelong,baxter98theoretical,baxter00model} with gradient ascent \cite{finn17modelagnostic}.

A classic approach to Bayesian bandits is to design Bayes optimal policies \cite{gittins79bandit,gittins11multiarmed}, which take a simple form for specific priors $\cP$. Our approach is more general, since it makes minimal assumptions about $\cP$ and optimized policies. It is also more computationally efficient and easier to parallelize. However, we lose guarantees on Bayes optimality. Another line of work \cite{russo14learning,wen15efficient,russo16information} bounds the Bayes regret of classic bandit policies. These policies also have instance-dependent regret bounds and thus are more conservative than our work, where we directly optimize the Bayes reward.

Overall, our aim is to make learning of bandit policies as straightforward as applying gradient descent to supervised learning problems. We take the following steps toward this goal. First, we carefully formulate the problem of policy-gradient optimization of the Bayes reward of bandit policies. Second, we derive the reward gradient and propose novel baseline subtraction methods that reduce the variance of its empirical estimate. These methods are tailored to the bandit structure of our problem and are critical to making our approach practical. Third, we show how to differentiate several softmax bandit policies: \expthree, \softelim, and neural networks with a softmax output layer. \softelim is a new algorithm where the probability of pulling an arm is directly parameterized. We prove that its $n$-round regret is sublinear in $n$ for any $K$-armed bandit, as in \ucb \cite{auer02finitetime} and Thompson sampling (\ts) \cite{thompson33likelihood,agrawal12analysis}. However, unlike \ucb and \ts, \softelim is easy to optimize. Finally, we evaluate our methodology empirically on a range of bandit problems, highlighting the versatility of our approach. We also show that neural network policies can learn interesting biases encoded in the prior distribution $\cP$.

%% file: Setting.tex

\section{Setting}
\label{sec:setting}

We define $[n] = \set{1, \dots, n}$. A \emph{Bayesian multi-armed bandit} \cite{gittins79bandit,berry85bandit} is an online learning problem where the learning agent interacts with problem instances that are drawn i.i.d.\ from a known prior distribution. Let $K$ be the number of arms, $n$ be the number of rounds, and $\cP$ be a \emph{prior distribution} over \emph{problem instances}. Each instance $P$ is a joint probability distribution over arm rewards with support $[0, 1]^K$. Let $Y_{i, t}$ be the reward of arm $i \in [K]$ in round $t \in [n]$ and $Y_t = (Y_{1, t}, \dots, Y_{K, t})$ be the vector of all rewards in round $t$. Before the agent starts interacting, we sample $P \sim \cP$ and $Y_t \sim P$ for all $t \in [n]$. Then, in each round round $t$, the agent \emph{pulls} arm $I_t \in [K]$ and gains its \emph{reward} $Y_{I_t, t}$. The agent knows $\cP$ but not the realized instance $P$.

We define $I_{i : j} = (I_i, \dots, I_j)$ and $Y_{i : j} = (Y_i, \dots, Y_j)$, with the corresponding $n$-round quantities being $I = I_{1 : n}$ and $Y = Y_{1 : n}$. Let $H_t = (I_1, \dots, I_t, Y_{I_1, 1}, \dots, Y_{I_t, t})$ be the \emph{history} of the learning agent in the first $t$ rounds, its pulled arms and rewards. The agent implements a \emph{randomized policy}. We denote by
\begin{align}
  p_\theta(i \mid H_{t - 1})
  \label{eq:policy}
\end{align}
the probability of pulling arm $i$ in round $t$ conditioned on history $H_{t - 1}$ up to that round. The policy is parameterized by $\theta \in \Theta$, where $\Theta$ is the space of feasible parameters. Thus $I_t \sim p_\theta(\cdot \mid H_{t - 1})$.

The \emph{$n$-round Bayes reward} of policy $\theta$ is $r(n; \theta) = \E{\sum_{t = 1}^n Y_{I_t, t}}$, where the expectation is over instances $P$, reward realizations $Y$, and arm choices $I_t$. The goal of the agent is to learn a policy $\theta_\ast = \argmax_{\theta \in \Theta} r(n; \theta)$ that maximizes the Bayes reward. This is equivalent to minimizing the \emph{$n$-round Bayes regret},
\begin{align}
  R(n; \theta)
  = \E{\sum_{t = 1}^n Y_{i_\ast(P), t} - \sum_{t = 1}^n Y_{I_t, t}}\,,
  \label{eq:bayes regret}
\end{align}
where $i_\ast(P) = \argmax_{i \in [K]} \condE{Y_{i, 1}}{P}$ is the best arm in problem instance $P$.

%% file: Optimization.tex

\section{Policy Optimization}
\label{sec:policy optimization}

We develop \gradband (\cref{alg:gradband}), an iterative gradient-based algorithm for optimizing bandit policies. \gradband is initialized with policy $\theta_0 \in \Theta$. At iteration $\ell$, the previous policy $\theta_{\ell - 1}$ is updated by gradient ascent using $\hat{g}(n; \theta_{\ell - 1})$, an empirical estimate of the reward gradient, $\nabla_\theta r(n; \theta_{\ell - 1})$, at $\theta_{\ell - 1}$. We compute $\hat{g}(n; \theta_{\ell - 1})$ by running $\theta_{\ell - 1}$ on $m$ instances sampled from $\cP$. We denote the $j$-th instance by $P^j$, its realized rewards by $Y^j \in [0, 1]^{K \times n}$, and its pulled arms by $I^j \in [K]^n$. The per-iteration time complexity of \gradband is $m K n$, since we sample $m$ problem instances from $\cP$ with horizon $n$ and $K$ arms, and run a policy in each.

Interestingly, \gradband does not require knowledge of $\cP$ nor it needs the sampled problem instances $(P^j)_{j = 1}^m$. This is because the computation of $\hat{g}(n; \theta_{\ell - 1})$ only requires realized rewards $(Y^j)_{j = 1}^m$ and pulled arms $(I^j)_{j = 1}^m$. So our assumption that $\cP$ is known merely simplifies the exposition.

\gradband is simple and general, because it makes no strong assumptions on the class of optimized policies, beyond the existence of $\nabla_\theta r(n; \theta)$. However, since $r(n; \theta)$ is a complex function of the adaptive policy $\theta$ and $\cP$, it is unclear if gradient ascent can ever converge to the best policy in $\Theta$. We provide the first such guarantee for this type of learned bandit policies below.

\begin{restatable}[]{theorem}{concavity}
\label{thm:concavity} Let $\cP$ be a prior distribution over $2$-armed Gaussian bandits $P$ where $Y_{i, t} \sim \cN(\mu_i, 1)$ and $\mu_i = \condE{Y_{i, 1}}{P}$. Let the policy class be an \emph{explore-then-commit} policy \cite{langford08epochgreedy} with parameter $\theta \in [1, \floors{n / 2}]$, which explores each arm $\bar{\theta} = \floors{\theta} + Z$ times and $Z \sim \mathrm{Ber}(\theta - \floors{\theta})$. Then $r(n; \theta)$ is concave in $\theta$ for any horizon $n \geq 2$.
\end{restatable}

The claim is proved in \cref{sec:concavity proof}. The key insight is that $r(n; \theta)$ of the explore-then-commit policy in a $2$-armed Gaussian bandit has a closed form, differentiable with respect to $\theta$. The randomization in \cref{thm:concavity} is only needed to extend the policy to continuous exploration horizons $\theta$. Note that in this case \gradband enjoys the same convergence guarantees as gradient descent for convex functions.

\gradband is a meta-algorithm. To fully exploit its power, we must specify the policy class $\Theta$ and compute the empirical gradient $\hat{g}(n; \theta_{\ell - 1})$. In \cref{sec:reward gradient}, we derive the gradient and show how to reduce its variance. In \cref{sec:differentiable algorithms}, we study several differentiable bandit policies. Before we proceed, we relate our objective and algorithm design to prior work.

\textbf{Stochastic multi-armed bandits:} Our objective, the maximization of $\E{\sum_{t = 1}^n Y_{I_t, t}}$, differs from maximizing $\condE{\sum_{t = 1}^n Y_{I_t, t}}{P}$ in any problem instance $P$, which is standard in bandits \cite{lai85asymptotically,auer02finitetime,lattimore19bandit}. The latter objective is more demanding, as it requires optimizing equally for likely and unlikely instances $P \sim \cP$. Our objective is more appropriate when $\cP$ can be estimated from data and the average reward is preferred to guarding against worst-case failures.

\textbf{Bayesian bandits:} Early works on Bayesian bandits \cite{gittins79bandit,berry85bandit,gittins11multiarmed} focus on deriving Bayes optimal policies, which require specific conjugate priors $\cP$. We do not make any such assumptions on the form of $\cP$. However, we do lose Bayes optimality guarantees, as the optimal policy may not lie in the chosen policy class $\Theta$. Since \gradband differentiates policies, it can be computationally costly. Nevertheless, it is less costly and easier to parallelize than the computation of typical Bayes optimal policies (\cref{sec:policy optimization experiment}).

\textbf{Reinforcement learning:} Learning of policy $\theta$ is also an instance of \emph{reinforcement learning (RL)} \cite{sutton88learning}, where the \emph{state} in round $t$ is history $H_{t - 1}$, the \emph{action} is the pulled arm $I_t$, and the \emph{reward} is the reward of the pulled arm $Y_{I_t, t}$. The main challenge is that the number of dimensions in $H_t$ increases linearly with round $t$. So any RL method that solves this problem must introduce some structure to deal with the \emph{curse of dimensionality}. Since it is not clear what the shape of the value function might be, we opt for optimizing parametric bandit policies (\cref{sec:differentiable algorithms}) by policy gradients \cite{williams92simple}. The main novelty in our application of policy gradients are baseline subtraction techniques that are tailored to the bandit structure of our problem.

\begin{algorithm}[t]
  \caption{Gradient-based optimization of bandit policies.}
  \label{alg:gradband}
  \begin{algorithmic}[1]
    \State \textbf{Inputs:} Initial policy $\theta_0 \in \Theta$, number of iterations $L$, learning rate $\alpha$, and batch size $m$
    \Statex \vspace{-0.075in}
    \For{$\ell = 1, \dots, L$}
      \For{$j = 1, \dots, m$}
        \State Sample $P^j \sim \cP$; sample $Y^j \sim P^j$; and apply policy $\theta_{\ell - 1}$ to $Y^j$ and obtain $I^j$
      \EndFor
      \State Let $\hat{g}(n; \theta_{\ell - 1})$ be an estimate of $\nabla_\theta r(n; \theta_{\ell - 1})$ from $(Y^j)_{j = 1}^m$ and $(I^j)_{j = 1}^m$
      \State $\theta_\ell \gets \theta_{\ell - 1} +
      \alpha \, \hat{g}(n; \theta_{\ell - 1})$
    \EndFor
    \Statex \vspace{-0.075in}
    \State \textbf{Output:} Learned policy $\theta_L$
  \end{algorithmic}
\end{algorithm}

%% file: Gradient.tex

\section{Reward Gradient}
\label{sec:reward gradient}

For any policy $\theta$, the reward gradient $\nabla_\theta r(n; \theta)$ 
takes the following form.

\begin{restatable}[]{theorem}{gradient}
\label{thm:gradient} For all rounds $t \in [n]$, let $b_t: [K]^{t - 1} \times [0, 1]^{K \times n} \to \realset$ be any function of previous $t - 1$ pulled arms and all reward realizations. Then
\begin{align*}
  \nabla_\theta r(n; \theta)
  = \sum_{t = 1}^n \mathbb{E}\Bigg[
  \nabla_\theta \log p_\theta(I_t \mid H_{t - 1})
  \left(\sum_{s = t}^n Y_{I_s, s} - b_t(I_{1 : t - 1}, Y)\right)\Bigg]\,.
\end{align*}
\end{restatable}

The claim is proved in \cref{sec:gradient proof}. The collection of functions $b = (b_t)_{t = 1}^n$ in \cref{thm:gradient} is known as a \emph{baseline} \cite{williams92simple,sutton00policy}. The baseline does not change the gradient, since each $b_t$ is independent of future actions taken by policy $\theta$ starting at round $t$. This means that $b_t$ can depend on other quantities with this property, such as the problem instance $P$ and parameters $\theta$. For simplicity, we do not make any such dependence explicit in our notation. Baselines can often effectively reduce the variance of empirical gradients. The empirical gradient, for $m$ sampled instances in \gradband, is
\begin{align}
  \hat{g}(n; \theta)
  = \frac{1}{m} \sum_{j = 1}^m \sum_{t = 1}^n
  \nabla_\theta \log p_\theta(I_t^j \mid H_{t - 1}^j)
  \left(\sum_{s = t}^n Y_{I_s^j, s}^j - b_t(I_{1 : t - 1}^j, Y^j)\right)\,,
  \label{eq:empirical baseline gradient}
\end{align}
where $j$ indexes the $j$-th random experiment in \gradband.

Now we discuss three baselines. \emph{No baseline} is a trivial baseline $b_t^\textsc{none}(I_{1: t - 1}, Y) = 0$. This baseline performs poorly, even when learning bandit policies at short horizons (\cref{sec:policy optimization experiment}).

Our second baseline is $b_t^\textsc{opt}(I_{1 : t - 1}, Y) = \sum_{s = t}^n Y_{i_\ast(P), s}$, where $i_\ast(P)$ is the best arm in instance $P$, as defined in \eqref{eq:bayes regret}. This baseline is suitable for bandit policies with regret guarantees. Specifically, if the policy has a sublinear regret with a high probability for any $P$, $\sum_{t = 1}^n Y_{i_\ast(P), t} - Y_{I_t, t} = o(n)$ and thus $\sum_{s = t}^n b_t(I_{1 : t - 1}, Y) - Y_{I_s, s} = o(n)$ for any $s \in [n]$; both with a high probability for any $P$.

One limitation of $b^\textsc{opt}$ is that the best arm may be unknown, for instance if $\gradband$ was only given sampled realized rewards $Y^j$ but not sampled instances $P^j$. This motivates our third baseline, which is the reward of an independent run of policy $\theta$. Let $(J_t)_{t = 1}^n$ be the arms pulled in that run. Then $b_t^\textsc{self}(I_{1 : t - 1}, Y) = \sum_{s = t}^n Y_{J_s, s}$. Similarly to $b^\textsc{opt}$, $b^\textsc{self}$ is suitable for any policy that concentrates on a single arm over time. Unlike $b^\textsc{opt}$, it does not need to know the best arm.

%% file: Algorithms.tex

\section{Differentiable Algorithms}
\label{sec:differentiable algorithms}

Our work assumes that $\nabla_\theta \log p_\theta(I_t \mid H_{t - 1})$ in \cref{thm:gradient} exists, that the policy is differentiable. However, existing bandit algorithms do not seem to fit this paradigm. For instance, UCB algorithms \cite{auer02finitetime,dani08stochastic,abbasi-yadkori11improved} are not differentiable because $p(i \mid H_{t - 1}) \in \set{0, 1}$ is a step function. While TS \cite{thompson33likelihood,agrawal12analysis,agrawal13thompson} is randomized, $p(i \mid H_{t - 1})$ is induced by a hard maximization over random variables. Therefore, a unique gradient may not always exist. Even if it does, $p(i \mid H_{t - 1})$ does not have a closed form and thus its differentiation is expected to be computationally costly.

In the rest of this section, we introduce three softmax designs that can be differentiated analytically and derive a gradient for each of them. All gradients are conditioned on a fixed round $t$ and history $H_{t - 1}$. To simplify notation, we define $p_{i, t} = p_\theta(i \mid H_{t - 1})$. Note that the $\varepsilon$-greedy policy \cite{sutton98reinforcement} and Boltzmann exploration \cite{sutton98reinforcement,cesabianchi17boltzmann} are also differentiable, although we do not study them here.

\subsection{Algorithm \expthree}
\label{sec:exp3}

\expthree \cite{auer95gambling} is a non-stochastic bandit algorithm, where the probability of pulling arm $i$ in round $t$ is
\begin{align}
  \textstyle
  p_{i, t}
  = \theta / K +
  (1 - \theta) \exp[\eta S_{i, t}] \big/ \sum_{j = 1}^K \exp[\eta S_{j, t}]\,,
  \label{eq:exp3 distribution}
\end{align}
where $S_{i, t}$ are \emph{sufficient statistics} of arm $i$ in round $t$, $\eta$ is a learning rate, and $\theta$ is a parameter that guarantees sufficient exploration. The statistic $S_{i, t}$ is the estimated cumulative reward of arm $i$ in the first $t - 1$ rounds, $S_{i, t} = \sum_{\ell = 1}^{t - 1} \I{I_\ell = i} p_{i, \ell}^{-1} Y_{i, \ell}$. When rewards are $[0, 1]$, $\expthree$ has $O(\sqrt{n K})$ regret for $\eta = \theta / K$ and $\theta = \min \set{1, \sqrt{K \log K} / \sqrt{(e - 1) n}}$. In this work, we \emph{optimize} the choice of $\theta$ using policy gradients. When $\eta$ is set as above, we get the following gradient.

\begin{restatable}[]{lemma}{expthreederivative}
\label{lem:exp3 derivative} Define $p_{i, t}$ as in \eqref{eq:exp3 distribution}.  Let $\eta = \theta / K$, $V_{i, t} = \exp[\theta S_{i, t} / K]$, and $V_t = \sum_{j = 1}^K V_{j, t}$. Then
\begin{align*}
  \nabla_\theta \log p_{i, t} =
  \frac{1}{p_{i, t}} \left[\frac{V_{i, t}}{V_t}
  \left[(1 - \theta) \left[\frac{S_{i, t}}{K} -
  \sum_{j = 1}^K \frac{V_{j, t}}{V_t} \frac{S_{j, t}}{K}\right] - 1\right] +
  \frac{1}{K}\right]\,.
\end{align*}
\end{restatable}

The claim is proved in \cref{sec:lemmas}. Although \expthree is differentiable, it is conservative in stochastic problems, even after we optimize $\theta$. Therefore, we propose a new algorithm \softelim.

\subsection{Algorithm \softelim}
\label{sec:softelim}

Our bandit algorithm works as follows. Each arm is initially pulled once. Let $\hat{\mu}_{i, t}$ be the empirical mean of arm $i$ after $t$ rounds and $T_{i, t}$ be the number of pulls of arm $i$ after $t$ rounds. Then in round $t > K$, arm $i$ is pulled with probability
\begin{align}
  \textstyle
  p_{i, t}
  = \exp[- S_{i, t} / \theta] \big/ \sum_{j = 1}^K \exp[- S_{j, t} / \theta]\,,
  \label{eq:softelim distribution}
\end{align}
where $S_{i, t} = 2 \, (\max_{j \in [K]} \hat{\mu}_{j, t - 1} - \hat{\mu}_{i, t - 1})^2 T_{i, t - 1}$ is the statistic associated with arm $i$ and $\theta > 0$ is a tunable \emph{exploration parameter}. Since $S_{i, t} \geq 0$, higher values of $\theta$ lead to more exploration. Also note that $\exp[- S_{i, t} / \theta] \in [0, 1]$. Therefore, our algorithm can be viewed as \say{soft} elimination \cite{auer10ucb} of arms with low empirical means. So we call it \softelim.

\softelim has two important properties. First, an arm is unlikely to be pulled if it has been pulled \say{often} and its empirical mean is low relative to the highest mean. Second, when a suboptimal arm has been pulled \say{often} and has the highest empirical mean, the optimal arm is pulled proportionally to how much its empirical mean deviates from the actual mean. This is why $\exp[- S_{i, t} / \theta]$ resembles the upper bound in Hoeffding's inequality. This latter property implies optimism.

Since $\log p_{i, t} = - \theta^{-1} S_{i, t} - \log \sum_{j = 1}^K \exp[- S_{j, t} / \theta]$, we have
\begin{align*}
  \textstyle
  \nabla_\theta \log p_{i, t}
  = \theta^{-2} \left(S_{i, t} -
  \sum_{j = 1}^K S_{j, t} \exp[- S_{j, t} / \theta] \big/
  \sum_{j = 1}^K \exp[- S_{j, t} / \theta]\right)\,.
\end{align*}
Therefore, \softelim can be easily differentiated and optimized by \gradband. \softelim also has a sublinear regret in any problem instance, as we show below.

\begin{restatable}[]{theorem}{regretbound}
\label{thm:regret bound} Let the \emph{expected $n$-round regret} of \softelim with parameter $\theta$ in problem instance $P$ be $R(n, P; \theta)$. Let $P$ be any $K$-armed bandit where arm $1$ is optimal, that is $\mu_1 > \max_{i > 1} \mu_i$. Let $\Delta_i = \mu_1 - \mu_i$ and $\theta = 8$. Then $R(n, P; \theta) \leq \sum_{i = 1}^K (2 e + 1) \left(16 \Delta_i^{-1} \log n + \Delta_i\right) + 5 \Delta_i$.
\end{restatable}

\cref{thm:regret bound} is proved in \cref{sec:analysis}, which also includes an informal argument. Note the our bound has the same standard dependence on gaps $\Delta_i$ and $\log n$ as \ucb \cite{auer02finitetime}. Thus it is near optimal.

\subsection{Recurrent Neural Network}
\label{sec:rnn}

Now we take designs \eqref{eq:exp3 distribution} and \eqref{eq:softelim distribution} a step further. Both are softmax on hand-crafted features, which facilitate theoretical analysis. We attempt to \emph{learn the features} using a \emph{recurrent neural network (RNN)}. The RNN works as follows. In round $t$, it takes arm $I_t$ and reward $Y_{I_t, t}$ as inputs, updates its state $\mathbf{s}_t$, and outputs the probability $p_{i, t + 1}$ of pulling each arm $i$ in the next round. That is,
\begin{align*}
  \textstyle
  \mathbf{s}_t
  = \textsc{RNN}_\Phi(\mathbf{s}_{t - 1}, (I_t, Y_{I_t, t}))\,, \quad
  p_{i, t + 1}
  = \exp[\mathbf{w}_i\T \mathbf{s}_t] \big/ 
  \sum_{j = 1}^K \exp[\mathbf{w}_j\T \mathbf{s}_t]\,.
\end{align*}
The optimized parameters $\theta = (\Phi, \set{\mathbf{w}_i}_{i = 1}^K)$ are the RNN parameters $\Phi$ and per-arm parameters $\mathbf{w}_j$. The aim for the RNN is to learn to track suitable sufficient statistics through its internal state $\mathbf{s}_t$. That state is initialized at $\mathbf{s}_0 = \mathbf{0}$. Our RNN is an LSTM \cite{hochreiter97long} with a $d$-dimensional latent state. We assume that the rewards are Bernoulli. The details of our implementation are in \cref{sec:rnn implementation}.

%% file: Experiments.tex

\section{Experiments}
\label{sec:experiments}

We conduct four experiments to demonstrate the generality and efficacy of our approach to learning bandit policies. In \cref{sec:reward gradient experiment}, we study the reward gradient and its variance in a simple problem. In \cref{sec:policy optimization experiment}, we optimize \expthree and \softelim policies on the same problem. In \cref{sec:more complex problems}, we study more complex bandit problems. In \cref{sec:rnn experiment}, we optimize RNN policies. The performance of policies is measured using the Bayes regret instead of the Bayes reward, since it offers a direct indication how close to optimal a policy is. Note that optimizing either optimizes the other. The regret is estimated from $1\,000$ i.i.d.\ samples from $\cP$, which are independent of the training samples used by \gradband. The shaded areas in plots show standard errors.

\subsection{Reward Gradient}
\label{sec:reward gradient experiment}

Our first experiment is on a Bayesian bandit with $K = 2$ arms. The first prior $\cP$ is simple and assigns probability $0.5$ to each of two bandit instances, with means $\mu = (0.6, 0.4)$ and $\mu = (0.4, 0.6)$. The reward distributions are Bernoulli and the horizon is $n = 200$ rounds.

The Bayes regret of \expthree and \softelim, as a function of their parameter $\theta$, is shown in \cref{fig:gradients}a. Both are unimodal in $\theta$ and suitable for optimization by \gradband. \softelim has a lower regret than \expthree for all $\theta$. In fact, the minimum regret of \expthree is greater than that of \softelim without tuning ($\theta = 1$). The reward gradients of \expthree and \softelim are reported in Figures \ref{fig:gradients}b and \ref{fig:gradients}c, respectively. We observe that baselines $b^\textsc{opt}$ and $b^\textsc{self}$ lead to orders of magnitude lower variance than no baseline $b^\textsc{none}$. The variance of \softelim gradients with $b^\textsc{opt}$ and $b^\textsc{self}$ is comparable, while the variance of \expthree gradients with $b^\textsc{self}$ is two orders of magnitude lower for higher values of $\theta$.

\subsection{Policy Optimization}
\label{sec:policy optimization experiment}

In the second experiment, we apply \expthree and \softelim to the problem in \cref{sec:reward gradient experiment}. The policies are optimized by \gradband using $\theta_0 = 1$, $L = 100$ iterations, learning rate $\alpha = c^{-1} L^{- \frac{1}{2}}$, and batch size $m = 1000$. The constant $c$ is chosen automatically so that $\norm{\hat{g}(n; \theta_0)} \leq c$ holds with a high probability, to avoid manual learning rate tuning in our experiments. We implement \gradband in TensorFlow on $112$ cores and with $392$ MB RAM.

\begin{figure*}[t]
  \centering
  \includegraphics[width=5.4in]{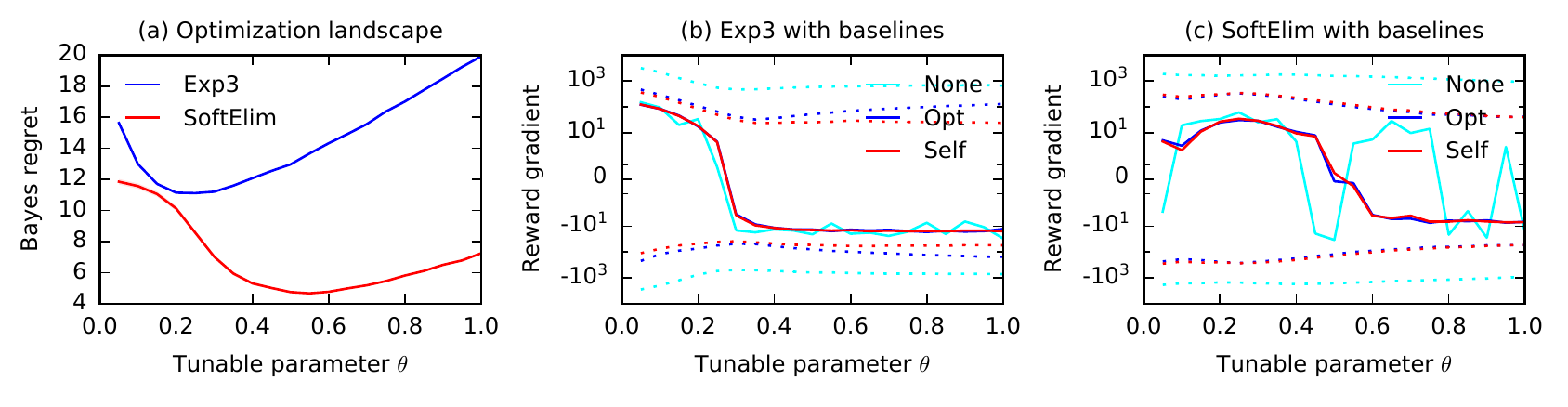} \\
  \vspace{-0.1in}
  \caption{The Bayes regret of \expthree and \softelim, and the corresponding reward gradients. In the last two plots, the solid lines are estimated reward gradients from $m = 10\,000$ runs and the dotted lines mark high-probability regions of empirical gradients, for $m = 1$.}
  \label{fig:gradients}
\end{figure*}

\begin{figure*}[t]
  \centering
  \includegraphics[width=5.4in]{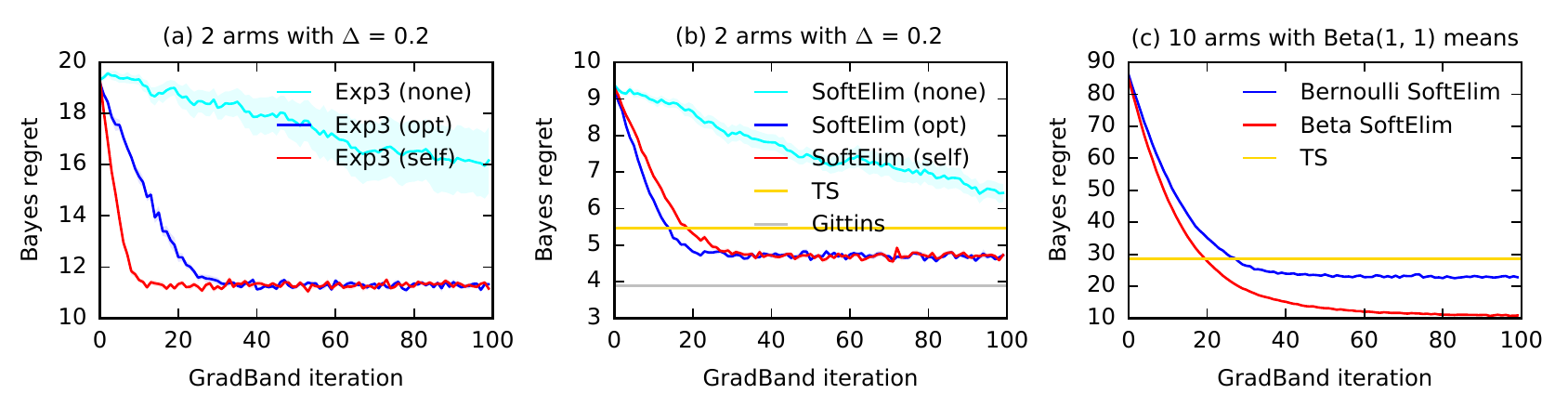} \\
  \vspace{-0.1in}
  \caption{The Bayes regret of \expthree and \softelim policies, as a function of \gradband iterations. We report the average over $10$ runs.}
  \label{fig:policy optimization}
\end{figure*}

In \cref{fig:policy optimization}a, we optimize \expthree with all baselines. With $b^\textsc{self}$, \gradband learns a near-optimal policy in fewer than $10$ iterations. This is consistent with \cref{fig:gradients}b, where $b^\textsc{self}$ has the least variance. In \cref{fig:policy optimization}b, we optimize \softelim with all baselines. The performance with $b^\textsc{opt}$ and $b^\textsc{self}$ is comparable. This consistent with \cref{fig:gradients}c, where the variances of $b^\textsc{opt}$ and $b^\textsc{self}$ are comparable. We conclude that $b^\textsc{self}$ is the best baseline overall and use it in all remaining experiments.

To assess the quality of our learned policies, we compare them to four well-known bandit policies: \ucb \cite{auer02finitetime}, Bernoulli \ts \cite{agrawal12analysis} with $\mathrm{Beta}(1, 1)$ prior, \ucbv \cite{audibert09exploration}, and the Gittins index \cite{gittins79bandit}. These benchmarks are ideal points of comparison: (i) \ucb is arguably the most popular bandit algorithm for $[0, 1]$ rewards. (ii) Bernoulli \ts is near-optimal for Bernoulli rewards, which we use in most experiments. We use randomized Bernoulli rounding \cite{agrawal12analysis} to apply \ts to $[0, 1]$ rewards. (iii) \ucbv adapts the sub-Gaussian parameter of its reward distributions based on past observations. This is similar to our optimization of $\theta$ in \softelim. (iv) The Gittins index gives the optimal solution to our problem, if the arm means were drawn i.i.d.\ from $\mathrm{Beta}(1, 1)$. Finally, we also use the Dopamine \cite{dopamine} implementation of DQN \cite{mnih13playing} where the state is a concatenation of the following statistics for each arm: the number of observed ones, the number of observed zeros, the logarithm of both counts incremented by $1$, the empirical mean, and a constant bias term.

The Bayes regret of our benchmarks is $9.95 \pm 0.03$ (\ucb), $5.47 \pm 0.05$ (\ts), $15.79 \pm 0.03$ (\ucbv), $3.89 \pm 0.07$ (Gittins index), and $16.81 \pm 1.05$ (DQN). The regret of \softelim is $4.75$, and falls between those of \ts and the Gittins index. We conclude that tuned \softelim outperforms a strong baseline, \ts; and performs almost as well as the Gittins index. We note that the Gittins index provides the optimal solution in limited settings, like Bernoulli bandits, but even there it is computationally costly. For instance, our computation of the Gittins index for horizon $n = 200$ took almost two days. In comparison, tuning of \softelim by \gradband takes about $20$ seconds.

Now we discuss failures of some benchmarks. \ucbv fails because its variance optimism induces too much initial exploration. This is harmful for the somewhat short horizons used in our experiments. DQN policies are unstable and require significant tuning to learn policies that outperform random actions; and still perform poorly. This stands in a stark contrast with the simplicity of \gradband, which learns near-optimal policies using gradient ascent. In the remaining experiments, we only discuss the most competitive benchmarks, the Gittins index and \ts. In \cref{sec:supplementary experiments}, we report the results for all benchmark bandit algorithms.

\subsection{More Complex Problems}
\label{sec:more complex problems}

In the third experiment, we apply \gradband to two more complex problems. In both, the number of arms is $K = 10$ and the mean reward of arm $i$ is $\mu_i \sim \mathrm{Beta}(1, 1)$. In the first, $Y_{i, t} \sim \mathrm{Ber}(\mu_i)$. In the second, $Y_{i, t} \sim \mathrm{Beta}(v \mu_i, v (1 - \mu_i))$ where $v = 4$ controls the variance of rewards. The horizon is $n = 1\,000$ rounds.

The regret of our policies is reported in \cref{fig:policy optimization}c. In the Bernoulli problem, the regret of tuned \softelim is less than $25$. By comparison, the regret of \ts is $28.57 \pm 0.45$. In the beta problem, the regret of tuned \softelim is close to $10$. The regret of \ts remains the same and is roughly three times that of \softelim. The poor performance of \ts is due to the Bernoulli rounding, which replaces low-variance beta rewards with high-variance Bernoulli rewards.

\subsection{RNN Policies}
\label{sec:rnn experiment}

\begin{figure*}[t]
  \centering
  \includegraphics[width=5.4in]{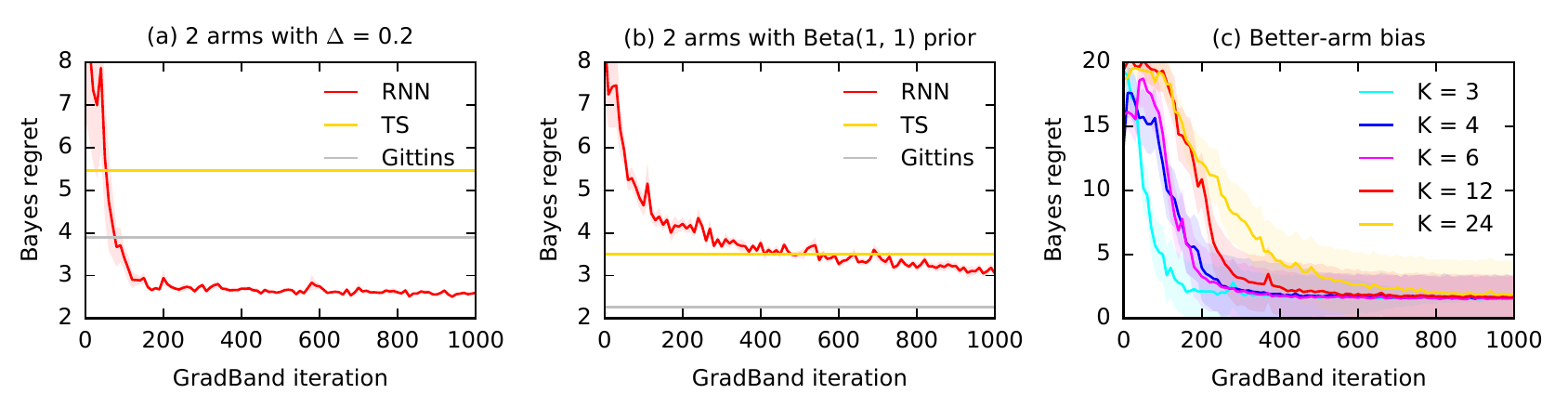} \\
  \vspace{-0.1in}
  \caption{The Bayes regret of RNN policies, as a function of \gradband iterations. We report the average over $10$ runs in the first two plots and the median in the last. The median excludes a few failed runs, which would skew the average.}
  \label{fig:rnn}
\end{figure*}

Our preliminary experiments showed that learning of RNN policies (\cref{sec:rnn}) over long horizons ($n = 200$) is challenging if we use our variance reduction baselines (\cref{sec:reward gradient}) alone. To mitigate this, we propose the use of \emph{curriculum learning} \cite{bengio09curriculum} to further reduce variance. The key idea is to apply \gradband successively to problems with increasing horizons. In this experiment, we consider a simple instance of this idea with two horizons: $n' = 20$ and $n = 200$. First, we optimize the RNN policy using \gradband at horizon $n'$. Then we take the learned policy and use it as the initial policy for \gradband optimization at horizon $n$. The number of \gradband iterations is $L = 1000$. We did not make any attempt to optimize this scheme.

The results from the second optimization phase are reported in \cref{fig:rnn}. \cref{fig:rnn}a shows learning of an RNN policy for the problem in \cref{sec:reward gradient experiment}. That policy outperforms both \ts and the Gittins index. This does not contradict theory, as the Gittins index is not Bayes optimal in this problem. In \cref{fig:rnn}b, we consider a variant of this problem where arm means are drawn i.i.d.\ from $\mathrm{Beta}(1, 1)$. The Gittins index is Bayes optimal in this problem, and so our learned RNN policy naturally does not outperform it. Nevertheless, it has a lower regret than \ts.

In the final experiment, we have a $K$-armed Bayesian bandit with Bernoulli rewards. The prior $\cP$ is over two problem instances, $\mu = (0.6, 0.9, 0.7, 0.7, \dots, 0.7)$ and $\mu = (0.2, 0.7, 0.9, 0.7, \dots, 0.7)$, which are equally likely. This problem has an interesting structure. The problem instance, and thus the optimal arm, can be identified by pulling arm $1$. Arms $4$ and beyond are \emph{distractors}. Our RNN policies do not learn this exact structure; but they learn another strategy specialized to this problem. The strategy pulls only arms $2$ or $3$, since these are the only arms that can be optimal. Thus, the RNN successfully learns to ignore the distractors. As a result, the Bayes regret of our policies (\cref{fig:rnn}c) does not increase with $K$. This would not happen with classic bandit algorithms.

%% file: RelatedWork.tex

\section{Related Work}
\label{sec:related work}

The regret of bandit algorithms can be reduced by tuning \cite{vermorel05multiarmed,maes12metalearning,kuleshov14algorithms,hsu19empirical}. None of these papers used policy gradients, neural network policies, or even the sequential structure of $n$-round rewards. \citet{duan16rl2} optimized a similar policy to \cref{sec:rnn} using an existing optimizer. This work does not formalize the objective clearly, relates it to Bayesian bandits, or studies policies that are provably sound (\cref{thm:concavity} and \cref{sec:softelim}). \citet{silver14deterministic} applied policy gradients to a continuous bandit problem with a quadratic cost function. Since the cost is convex in arms, this exploration problem is easier than with discrete arms.

Policy gradients in RL were proposed by \citet{williams92simple}, including the idea of baseline subtraction. Other early works on this topic are \citet{sutton00policy} and \citet{baxter01infinitehorizon}. Policy gradients tend to have a high variance and reducing it is an important research area \cite{greensmith04variance,munos06geometric,zhao11analysis,dick15policy,liu18actiondependent}. Our baselines differ from those in RL, in particular because our number of states $H_t$ is not small. The baseline $b^\textsc{opt}$ uses the fact that we have a bandit problem, and thus the best arm in hindsight. Both $b^\textsc{opt}$ and $b^\textsc{self}$ use the fact that we have access to all rewards, even of arms not pulled by policy $\pi$.

Our approach is an instance of meta-learning
\cite{thrun96explanationbased,thrun98lifelong}, where we learn from a sample of tasks to perform well on tasks drawn from the same distribution \cite{baxter98theoretical,baxter00model}. Meta-learning has been applied successfully in deep reinforcement learning (RL) \cite{finn17modelagnostic,finn18probabilistic,mishra18simple}. Sequential multitask learning \cite{caruana97multitask} was studied in multi-armed bandits by \citet{azar13sequential} and in contextual bandits by \citet{deshmukh17multitask}. In comparison, our setting is offline. A general template for sequential meta-learning was presented in \citet{ortega19metalearning}. This work is conceptual and does not study policy gradients.

\citet{maillard11thesis} proposed \softelim with $\theta = 1$ and bounded the number of pulls of a suboptimal arm in Theorem 1.10. The bound has a large $O(K \Delta^{-4})$ constant, which does not seem easy to eliminate. We introduce $\theta$ and have a tighter analysis (\cref{thm:regret bound}) with a $O(1)$ constant, for $\theta = 8$. Also note that \softelim is not very competitive with \ts without tuning. Therefore, this approach have not received much attention in the past, and this is the first work that makes it practical. The design of \softelim resembles Boltzmann exploration \cite{sutton98reinforcement,cesabianchi17boltzmann} and \expthree (\cref{sec:exp3}). The key difference is in how $S_{i, t}$ is chosen. In \expthree and Boltzmann exploration, $S_{i, t}$ only depends on the history of arm $i$. In \softelim, $S_{i, t}$ depends on all arms and makes \softelim sufficiently optimistic.

%% file: Conclusions.tex

\section{Conclusions}
\label{sec:conclusions}

We take first steps towards understanding policy-gradient optimization of bandit policies. Our work addresses two main challenges of this problem. First, we derive the reward gradient of optimized policies and show how to estimate it efficiently from a sample. Second, we propose differentiable bandit policies that can outperform state-of-the-art baselines after optimization. Our experiments highlight the simplicity and generality of our approach. We also show that neural network policies can learn interesting biases.

We leave open several questions of interest. First, the design of \softelim can be generalized to structured problems, which we plan to pursue next. The key insight that permits generalization is that $S_{i, t} \propto (\max_{j \in [K]} \hat{\mu}_{j, t - 1} - \hat{\mu}_{i, t - 1})^2 T_{i, t - 1}$ is a ratio of two terms, the squared empirical suboptimality gap of arm $i$ and the variance of the mean reward estimate of arm $i$, which is $1 / T_{i, t - 1}$ in this case. Such quantities can be computed in linear models, for instance. Second, we find that the variance of empirical reward gradients can be high, especially in RNN policies. So any progress in variance reduction would be of a major importance. Finally, except for \cref{thm:concavity}, we are unaware of other algorithm-bandit instance pairs where the Bayes reward is concave in optimized parameters, and thus gradient ascent leads to optimal solutions. Our empirical results (\cref{fig:gradients}a) suggest that such pairs may be common. Convergence guarantees for optimizing softmax bandit policies may be possible in the future, inspired by recent advances in analyzing policy gradients in RL \cite{agarwal19optimality,mei20global}.

%% file: ConcavityProof.tex

\section{Proof of \cref{thm:concavity}}
\label{sec:concavity proof}

We show that the $n$-round Bayes reward of a randomized explore-then-commit policy in $2$-armed Gaussian bandits is concave in the exploration horizon of the policy.

\concavity*
\begin{proof}
We start with the \emph{explore-then-commit} policy \cite{langford08epochgreedy}, which is parameterized by $\theta \in [\floors{n / 2}]$ and works as follows. In the first $2 \theta$ rounds, it explores and pulls each arm $\theta$ times. Let $\hat{\mu}_{i, \theta}$ be the average reward of arm $i$ after $\theta$ pulls. Then, if $\hat{\mu}_{1, \theta} > \hat{\mu}_{2, \theta}$, arm $1$ is pulled for the remaining $n - 2 \theta$ rounds. Otherwise arm $2$ is pulled.

Fix any problem instance $P \sim \cP$. Without loss of generality, let arm $1$ be optimal, that is $\mu_1 > \mu_2$. Let $\Delta = \mu_1 - \mu_2$. The key observation is that the expected $n$-round reward in problem instance $P$ has a closed form
\begin{align}
  r(n, P; \theta)
  & = \mu_1 n - \Delta \left[\theta +
  \prob{\hat{\mu}_{1, \theta} < \hat{\mu}_{2, \theta}} (n - 2 \theta)\right]\,,
  \label{eq:concave reward}
\end{align}
where
\begin{align}
  \prob{\hat{\mu}_{1, \theta} < \hat{\mu}_{2, \theta}}
  & = \prob{\hat{\mu}_{1, \theta} - \hat{\mu}_{2, \theta} < 0}
  = \prob{\hat{\mu}_{1, \theta} - \hat{\mu}_{2, \theta} - \Delta < - \Delta}
  \nonumber \\
  & = \Phi\left(- \Delta \sqrt{\theta / 2}\right)
  = \frac{1}{\sqrt{2 \pi}}
  \int_{x = - \infty}^{- \Delta \sqrt{\theta / 2}} e^{- \frac{x^2}{2}} \dif x
  \label{eq:bad choice}
\end{align}
is the probability of committing to a suboptimal arm after the exploration phase. The third equality is from the fact that $\hat{\mu}_{1, \theta} - \hat{\mu}_{2, \theta} - \Delta \sim \cN(0, 2 / \theta)$, where $\Phi(x)$ is the cumulative distribution function of the standard normal distribution.

Our goal is to prove that $r(n, P; \theta)$ is concave in $\theta$. We rely on the following property of convex functions of a single parameter $x$. Let $f(x)$ and $g(x)$ be non-negative, decreasing, and convex in $x$. Then $f(x) g(x)$ is non-negative, decreasing, and convex in $x$. This follows from
\begin{align*}
  (f(x) g(x))'
  & = f'(x) g(x) + f(x) g'(x)\,, \\
  (f(x) g(x))''
  & = f''(x) g(x) + 2 f'(x) g'(x) + f(x) g''(x)\,.
\end{align*}
It is easy to see that \eqref{eq:bad choice} is non-negative, decreasing, and convex in $\theta$. The same is true for $n - 2 \theta$, under our assumption that $\theta \in [\floors{n / 2}]$. As a result, $\prob{\hat{\mu}_{1, \theta} < \hat{\mu}_{2, \theta}} (n - 2 \theta)$ is convex in $\theta$, and so is $\Delta [\theta + \prob{\hat{\mu}_{1, \theta} < \hat{\mu}_{2, \theta}} (n - 2 \theta)]$. Therefore, $\eqref{eq:concave reward}$ is concave in $\theta$. Finally, the Bayes reward is concave in $\theta$ because $r(n; \theta) = \E{r(n,  P; \theta)}$.

The last remaining issue is that parameter $\theta$ in the explore-then-commit policy cannot be optimized by \gradband, as it is discrete. To allow for optimization, we extend the explore-then-commit policy to continuous $\theta$ by randomized rounding.

The \emph{randomized explore-then-commit} policy is parameterized by continuous $\theta \in [1, \floors{n / 2}]$. The discrete $\bar{\theta}$ is chosen as $\bar{\theta} = \floors{\theta} + Z$, where $Z \sim \mathrm{Ber}(\theta - \floors{\theta})$. Then we execute the original policy with $\bar{\theta}$. The key property of the randomized policy is that its $n$-round Bayes reward is a piecewise linear interpolation of that of the original policy,
\begin{align*}
  (\ceils{\theta} - \theta) \, r(n; \floors{\theta}) +
  (\theta - \floors{\theta}) \, r(n; \ceils{\theta})\,.
\end{align*}
By definition, the above function is continuous and concave in $\theta$. Therefore, \gradband has the same guarantees for maximizing it as stochastic gradient descent on convex functions.
\end{proof}

%% file: GradientProof.tex

\section{Proof of \cref{thm:gradient}}
\label{sec:gradient proof}

We derive the gradient of the $n$-round Bayes reward below.

\gradient*
\begin{proof}
The proof has two parts. First, we show that
\begin{align}
  \nabla_\theta r(n; \theta)
  = \sum_{t = 1}^n \E{\nabla_\theta \log p_\theta(I_t \mid H_{t - 1})
  \sum_{s = t}^n Y_{I_s, s}}\,.
  \label{eq:gradient}
\end{align}
The $n$-round Bayes reward can be expressed as $r(n; \theta) = \E{\condE{\sum_{t = 1}^n Y_{I_t, t}}{Y}}$, where the outer expectation is over instances $P$ and their reward realizations $Y$, both of which are independent of $\theta$. Therefore,
\begin{align*}
  \nabla_\theta r(n; \theta)
  = \E{\sum_{t = 1}^n \nabla_\theta \condE{Y_{I_t, t}}{Y}}\,.
\end{align*}
In the inner expectation, the only randomness is due to the pulled arms. Therefore, for any $t \in [n]$, we have
\begin{align*}
  \condE{Y_{I_t, t}}{Y}
  = \sum_{i_{1 : t}} \condprob{I_{1 : t} = i_{1 : t}}{Y} Y_{i_t, t}\,.
\end{align*}
The key to our derivations is that the joint probability distribution over pulled arms in the first $t$ rounds, conditioned on $Y$, decomposes as
\begin{align}
  \condprob{I_{1 : t} = i_{1 : t}}{Y}
  = \prod_{s = 1}^t \condprob{I_s = i_s}{I_{1 : s - 1} = i_{1 : s - 1}, Y}\,,
  \label{eq:chain rule}
\end{align}
by the chain rule of probabilities. Since the policy does not act based on future rewards, we have for any $s \in [n]$ that
\begin{align}
  \condprob{I_s = i_s}{I_{1 : s - 1} = i_{1 : s - 1}, Y}
  = p_\theta(i_s \mid i_{1 : s - 1}, Y_{i_1, 1}, \dots, Y_{i_{s - 1}, s - 1})\,.
  \label{eq:policy equivalence}
\end{align}
Finally, we use that $\nabla_\theta f(\theta) = f(\theta) \nabla_\theta \log f(\theta)$ holds for any non-negative differentiable $f$. This identity is known as the \emph{score-function identity} \cite{aleksandrov68stochastic} and is the basis of all policy-gradient methods. We apply it to $\condE{Y_{I_t, t}}{Y}$ and obtain
\begin{align*}
  \nabla_\theta \condE{Y_{I_t, t}}{Y}
  & = \sum_{i_{1 : t}}
  Y_{i_t, t} \nabla_\theta \condprob{I_{1 : t} = i_{1 : t}}{Y} \\
  & = \sum_{i_{1 : t}} Y_{i_t, t} \, \condprob{I_{1 : t} = i_{1 : t}}{Y}
  \nabla_\theta \log \condprob{I_{1 : t} = i_{1 : t}}{Y} \\
  & = \sum_{s = 1}^t \condE{Y_{I_t, t}
  \nabla_\theta \log p_\theta(I_s \mid H_{s - 1})}{Y}\,,
\end{align*}
where the last equality follows from \eqref{eq:chain rule} and \eqref{eq:policy equivalence}. Now we chain all equalities to obtain the \emph{reward gradient}
\begin{align*}
  \nabla_\theta r(n; \theta)
  = \sum_{t = 1}^n \sum_{s = 1}^t \E{Y_{I_t, t}
  \nabla_\theta \log p_\theta(I_s \mid H_{s - 1})}
  = \sum_{t = 1}^n \E{\nabla_\theta \log p_\theta(I_t \mid H_{t - 1})
  \sum_{s = t}^n Y_{I_s, s}}\,.
\end{align*}
This concludes the first part of the proof.

Now we argue that $b_t$ does not change anything. Since $b_t$ depends only on $I_{1 : t - 1}$ and $Y$,
\begin{align*}
  \E{b_t(I_{1 : t - 1}, Y) \nabla_\theta \log p_\theta(I_t \mid H_{t - 1})}
  = \E{b_t(I_{1 : t - 1}, Y)
  \condE{\nabla_\theta \log p_\theta(I_t \mid H_{t - 1})}{I_{1 : t - 1}, Y}}\,.
\end{align*}
Now note that
\begin{align*}
  \condE{\nabla_\theta \log p_\theta(I_t \mid H_{t - 1})}{I_{1 : t - 1}, Y}
  & = \sum_{i = 1}^K \condprob{I_t = i}{I_{1 : t - 1}, Y}
  \nabla_\theta \log p_\theta(i \mid H_{t - 1}) \\
  & = \sum_{i = 1}^K p_\theta(i \mid H_{t - 1})
  \nabla_\theta \log p_\theta(i \mid H_{t - 1}) \\
  & = \nabla_\theta \sum_{i = 1}^K p_\theta(i \mid H_{t - 1})
  = 0\,.
\end{align*}
The last equality follows from $\sum_{i = 1}^K p_\theta(i \mid H_{t - 1}) = 1$, which is a constant independent of $\theta$. This concludes the proof.
\end{proof}

%% file: Analysis.tex

\section{Analysis of \softelim}
\label{sec:analysis}

First, we informally justify \softelim in \cref{sec:informal analysis}. The regret bound is stated and proved in \cref{sec:regret bound}.

\subsection{Informal Analysis}
\label{sec:informal analysis}

Fix any $2$-armed bandit where arm $1$ is optimal, that is $\mu_1 > \mu_2$. Let $\Delta = \mu_1 - \mu_2$. Fix any round $t$ by which arm $2$ has been pulled \say{often}, so that we get $T_{2, t - 1} = \Omega(\Delta^{-2} \log n)$ and $\hat{\mu}_{2, t - 1} \leq \mu_2 + \Delta / 3$ with high probability. Let
\begin{align*}
  \hat{\mu}_{\max, t}
  = \max \set{\hat{\mu}_{1, t}, \hat{\mu}_{2, t}}\,.
\end{align*}
Now consider two cases. First, when $\hat{\mu}_{\max, t - 1} = \hat{\mu}_{1, t - 1}$, by definition of $p_{1, t}$, arm $1$ is pulled with probability of at least $0.5$. Second, when $\hat{\mu}_{\max, t - 1} = \hat{\mu}_{2, t - 1}$, we have
\begin{align*}
  p_{1, t}
  = \exp[-2 (\hat{\mu}_{2, t - 1} - \hat{\mu}_{1, t - 1})^2 T_{1, t - 1}] p_{2, t}
  \geq \exp[-2 (\mu_1 - \hat{\mu}_{1, t - 1})^2 T_{1, t - 1}] p_{2, t}\,,
\end{align*}
where the last inequality holds with high probability, and follows from $\hat{\mu}_{1, t - 1} \leq \hat{\mu}_{2, t - 1} \leq \mu_2 + \Delta / 3 \leq \mu_1$. Thus, arm $1$ is pulled \say{sufficiently often} relative to arm $2$, proportionally to the deviation of $\hat{\mu}_{1, t - 1}$ from $\mu_1$.

As a consequence, \softelim eventually enters a regime in which arm $1$ has been pulled \say{often}, so that $T_{1, t - 1} = \Omega(\Delta^{-2} \log n)$ and $\hat{\mu}_{1, t - 1} \geq \mu_1 - \Delta / 3$ with high probability. Then $S_{1, t} = 0$ and $S_{2, t} = \Omega(\log n)$ hold with high probability, and arm $2$ is unlikely to be pulled.

\subsection{Regret Bound}
\label{sec:regret bound}

We bound the $n$-round regret of \softelim below.

\regretbound*
\begin{proof}
Each arm is initially pulled once. Therefore,
\begin{align*}
  R(n, P; \theta)
  = \sum_{i = 1}^K \Delta_i \left(\sum_{t = K + 1}^n \prob{I_t = i} + 1\right)\,.
\end{align*}
Now we decompose the probability of pulling each arm $i$ as
\begin{align*}
  \sum_{t = K + 1}^n \prob{I_t = i}
  = & \sum_{t = K + 1}^n \prob{I_t = i, T_{i, t - 1} \leq m} + {} \\
  & \sum_{t = K + 1}^n \prob{I_t = i, T_{i, t - 1} > m, T_{1, t - 1} \leq m} + {} \\
  & \sum_{t = K + 1}^n \prob{I_t = i, T_{i, t - 1} > m, T_{1, t - 1} > m}\,,
\end{align*}
where $m$ is chosen later. In the rest of the proof, we bound each above term separately. To simplify notation, use $\gamma = 1 / \theta$ in instead of $\theta$.

\subsection{Upper Bound on Term $1$}
\label{sec:softelim term 1}

Fix suboptimal arm $i$. Since $T_{i, t} = T_{i, t - 1} + 1$ on event $I_t = i$ and arm $i$ is initially pulled once, we have
\begin{align}
  \sum_{t = K + 1}^n \prob{I_t = i, T_{i, t - 1} \leq m}
  \leq m - 1\,.
  \label{eq:term 1}
\end{align}

\subsection{Upper Bound on Term $3$}
\label{sec:softelim term 3}

Fix suboptimal arm $i$ and round $t$. Let
\begin{align*}
  E_{1, t}
  = \set{\hat{\mu}_{1, t - 1} > \mu_1 - \frac{\Delta_i}{4}}\,, \quad
  E_{i, t}
  = \set{\hat{\mu}_{i, t - 1} < \mu_i + \frac{\Delta_i}{4}}\,,
\end{align*}
be the events that empirical means of arms $1$ and $i$, respectively, are \say{close} to their means. Then
\begin{align*}
  & \prob{I_t = i, T_{i, t - 1} > m, T_{1, t - 1} > m} \\
  & \quad \leq \prob{I_t = i, T_{i, t - 1} > m, E_{1, t}} +
  \prob{\bar{E}_{1, t}, T_{1, t - 1} > m} \\
  & \quad \leq \prob{I_t = i, T_{i, t - 1} > m, E_{1, t}, E_{i, t}} +
  \prob{\bar{E}_{1, t}, T_{1, t - 1} > m} +
  \prob{\bar{E}_{i, t}, T_{i, t - 1} > m}\,.
\end{align*}
Let $m = \ceils{16 \Delta_i^{-2} \log n}$. By the union bound and Hoeffding's inequality, we get
\begin{align*}
  \prob{\bar{E}_{1, t}, T_{1, t - 1} > m}
  & \leq \sum_{s = m + 1}^n
  \prob{\mu_1 - \hat{\mu}_{1, t - 1} \geq \frac{\Delta_i}{4}, \, T_{1, t - 1} = s}
  < n \exp\left[-2 \frac{\Delta_i^2}{16} m\right]
  = n^{-1}\,, \\
  \prob{\bar{E}_{i, t}, T_{i, t - 1} > m}
  & \leq \sum_{s = m + 1}^n
  \prob{\hat{\mu}_{i, t - 1} - \mu_i \geq \frac{\Delta_i}{4}, \, T_{i, t - 1} = s}
  < n \exp\left[-2 \frac{\Delta_i^2}{16} m\right]
  = n^{-1}\,.
\end{align*}
It follows that
\begin{align*}
  \prob{I_t = i, T_{i, t - 1} > m, T_{1, t - 1} > m}
  \leq \prob{I_t = i, T_{i, t - 1} > m, E_{1, t}, E_{i, t}} + 2 n^{-1}\,.
\end{align*}
Now note that $\hat{\mu}_{1, t - 1} - \hat{\mu}_{i, t - 1} \geq \Delta_i / 2$ on events $E_{1, t}$ and $E_{i, t}$. Let
\begin{align}
  \hat{\mu}_{\max, t - 1}
  = \max_{i \in [K]} \hat{\mu}_{i, t - 1}
  \label{eq:highest empirical mean}
\end{align}
be the highest empirical mean in round $t$. Since $\hat{\mu}_{\max, t - 1} \geq \hat{\mu}_{1, t - 1}$, we have $\hat{\mu}_{\max, t - 1} - \hat{\mu}_{i, t - 1} \geq \Delta_i / 2$. Therefore, on event $T_{i, t - 1} > m$, we get
\begin{align}
  p_{i, t}
  \leq \exp[-2 \gamma (\hat{\mu}_{\max, t - 1} - \hat{\mu}_{i, t - 1})^2 T_{i, t - 1}]
  \leq \exp\left[-2 \gamma \frac{\Delta_i^2}{4} m\right]
  \leq n^{-8 \gamma}\,.
  \label{eq:unlikely pull}
\end{align}
Finally, we chain all inequalities over all rounds and get that term $3$ is bounded as
\begin{align}
  \sum_{t = K + 1}^n \prob{I_t = i, T_{i, t - 1} > m, T_{1, t - 1} > m}
  \leq n^{1 - 8 \gamma} + 2\,.
  \label{eq:term 3}
\end{align}

\subsection{Upper Bound on Term $2$}
\label{sec:softelim term 2}

Fix suboptimal arm $i$ and round $t$. First, we apply Hoeffding's inequality to arm $i$, as in \cref{sec:softelim term 3}, and get
\begin{align*}
  \prob{I_t = i, T_{i, t - 1} > m, T_{1, t - 1} \leq m}
  & \leq \prob{I_t = i, T_{i, t - 1} > m, T_{1, t - 1} \leq m, E_{i, t}} + n^{-1} \\
  & = \E{p_{i, t} \I{T_{i, t - 1} > m, T_{1, t - 1} \leq m, E_{i, t}}} + n^{-1}\,.
\end{align*}
Let $\hat{\mu}_{\max, t - 1}$ be defined as in \eqref{eq:highest empirical mean}. Now we bound $p_{i, t}$ from above using $p_{1, t}$. We consider two cases. First, suppose that $\hat{\mu}_{\max, t - 1} > \mu_1 - \Delta_i / 4$. Then we have \eqref{eq:unlikely pull}. On the other hand, when $\hat{\mu}_{\max, t - 1} \leq \mu_1 - \Delta_i / 4$, we have
\begin{align}
  p_{i, t}
  = \frac{\exp[-2 \gamma (\hat{\mu}_{\max, t - 1} - \hat{\mu}_{i, t - 1})^2 T_{i, t - 1}]}
  {\exp[-2 \gamma (\hat{\mu}_{\max, t - 1} - \hat{\mu}_{1, t - 1})^2 T_{1, t - 1}]}
  p_{1, t}
  \leq \exp[2 \gamma (\mu_1 - \hat{\mu}_{1, t - 1})^2 T_{1, t - 1}]
  p_{1, t}\,.
  \label{eq:pi to p1}
\end{align}
It follows that
\begin{align*}
  p_{i, t}
  \leq \exp[2 \gamma (\mu_1 - \hat{\mu}_{1, t - 1})^2 T_{1, t - 1}] p_{1, t} +
  n^{-8 \gamma}\,,
\end{align*}
and we further get that
\begin{align*}
  & \E{p_{i, t} \I{T_{i, t - 1} > m, T_{1, t - 1} \leq m, E_{i, t}}} \\
  & \quad \leq \E{\exp[2 \gamma (\mu_1 - \hat{\mu}_{1, t - 1})^2
  T_{1, t - 1}] p_{1, t} \I{T_{1, t - 1} \leq m}} + n^{-8 \gamma} \\
  & \quad = \E{\exp[2 \gamma (\mu_1 - \hat{\mu}_{1, t - 1})^2
  T_{1, t - 1}] \I{I_t = 1, T_{1, t - 1} \leq m}} + n^{-8 \gamma}\,.
\end{align*}
With a slight abuse of notation, let $\hat{\mu}_{1, s}$ denote the average reward of arm $1$ after $s$ pulls. Then, since $T_{1, t} = T_{1, t - 1} + 1$ on event $I_t = 1$, we have
\begin{align*}
  \sum_{t = K + 1}^n \E{\exp[2 \gamma (\mu_1 - \hat{\mu}_{1, t - 1})^2
  T_{1, t - 1}] \I{I_t = 1, T_{1, t - 1} \leq m}}
  \leq \sum_{s = 1}^m \E{\exp[2 \gamma (\mu_1 - \hat{\mu}_{1, s})^2 s]}\,.
\end{align*}
Now fix the number of pulls $s$ and note that
\begin{align*}
  \E{\exp[2 \gamma (\mu_1 - \hat{\mu}_{1, s})^2 s]}
  & \leq \sum_{\ell = 0}^\infty
  \prob{\frac{\ell + 1}{\sqrt{s}} > \abs{\mu_1 - \hat{\mu}_{1, s}} \geq
  \frac{\ell}{\sqrt{s}}} \exp[2 \gamma (\ell + 1)^2] \\
  & \leq \sum_{\ell = 0}^\infty
  \prob{\abs{\mu_1 - \hat{\mu}_{1, s}} \geq \frac{\ell}{\sqrt{s}}}
  \exp[2 \gamma (\ell + 1)^2] \\
  & \leq 2 \sum_{\ell = 0}^\infty \exp[2 \gamma (\ell + 1)^2 - 2 \ell^2]\,,
\end{align*}
where the last step is by Hoeffding's inequality. The above sum can be easily bounded for any $\gamma < 1$. In particular, for $\gamma = 1 / 8$, the bound is
\begin{align*}
  \sum_{\ell = 0}^\infty \exp\left[\frac{(\ell + 1)^2}{4} - 2 \ell^2\right]
  \leq e^{\frac{1}{4}} + \sum_{\ell = 1}^\infty 2^{- \ell}
  \leq e\,.
\end{align*}
Now we combine all above inequalities and get that term $2$ is bounded as
\begin{align}
  \sum_{t = K + 1}^n \prob{I_t = i, T_{i, t - 1} > m, T_{1, t - 1} \leq m}
  \leq 2 e m + n^{1 - 8 \gamma} + 1\,.
  \label{eq:term 2}
\end{align}
Finally, we chain \eqref{eq:term 1}, \eqref{eq:term 3}, and \eqref{eq:term 2}; and use that $m \leq 16 \Delta_i^{-2} \log n + 1$.
\end{proof}

%% file: Lemmas.tex

\section{Technical Lemmas}
\label{sec:lemmas}

\expthreederivative*
\begin{proof}
First, we express the derivative of $\log p_{i, t}$ with respect to $\theta$ as
\begin{align*}
  \nabla_\theta \log p_{i, t}
  = \frac{1}{p_{i, t}} \nabla_\theta p_{i, t}
  = \frac{1}{p_{i, t}}
  \left[(1 - \theta) \nabla_\theta \frac{V_{i, t}}{V_t} -
  \frac{V_{i, t}}{V_t} + \frac{1}{K}\right]\,.
\end{align*}
Now note that
\begin{align*}
  \nabla_\theta \frac{V_{i, t}}{V_t}
  = \frac{1}{V_t} \nabla_\theta V_{i, t} + V_{i, t} \nabla_\theta \frac{1}{V_t}
  = \frac{V_{i, t} S_{i, t}}{V_t K} -
  \frac{V_{i, t}}{V_t^2} \sum_{j = 1}^K V_{j, t} \frac{S_{j, t}}{K}
  = \frac{V_{i, t}}{V_t} \left[\frac{S_{i, t}}{K} -
  \sum_{j = 1}^K \frac{V_{j, t}}{V_t} \frac{S_{j, t}}{K}\right]\,.
\end{align*}
This concludes the proof.
\end{proof}

%% file: Experiments2.tex

\section{Supplementary Experiments}
\label{sec:supplementary experiments}

The Bayes regret of baseline bandit algorithms in \cref{fig:policy optimization,fig:rnn} is reported in \cref{tab:baselines}.

\begin{table*}[t]
  \centering
  {\scriptsize
  \begin{tabular}{l|rrrrrr} \hline
    Figure & 2a-b & 2c (Bernoulli) & 2c (beta) & 3a & 3b \\ \hline
    Gittins index & $3.89 \pm 0.07$ & x & x &
    $3.89 \pm 0.07$ & $2.26 \pm 0.04$ \\
    \ts & $5.47 \pm 0.05$ & $28.06 \pm 0.45$ & $28.06 \pm 0.45$ &
    $5.47 \pm 0.05$ & $3.50 \pm 0.03$ \\
    \ucb & $9.95 \pm 0.03$ & $129.09 \pm 0.60$ & $129.09 \pm 0.60$ &
    $9.95 \pm 0.03$ & $8.52 \pm 0.03$ \\
    \ucbv & $15.79 \pm 0.03$ & $289.82 \pm 1.90$ & $276.07 \pm 1.65$ &
    $15.79 \pm 0.03$ & $19.03 \pm 0.10$ \\ \hline
  \end{tabular}
  }
  \caption{The Bayes regret of baseline bandit algorithms in \cref{fig:policy optimization,fig:rnn}. The crosses mark computationally-prohibitive experiments.}
  \label{tab:baselines}
\end{table*}

%% file: RNNDetails.tex

\section{RNN Implementation}
\label{sec:rnn implementation}

We carry out the RNN experiments using PyTorch framework.
In this paper, we restrict ourselves to binary 0/1 rewards.
For all experiments, our policy network is a single layer LSTM followed by LeakyRELU non-linearity and a fully connected layer. 
We use the fixed LSTM latent state dimension of 50, irrespective of numbers of arms.
The implementation of the policy network is provided in the code snippet below:

\begin{lstlisting}[language=Python, caption=Policy Network]
class RecurrentPolicyNet(nn.Module):
  def __init__(self, K=2, d=50):
    super(RecurrentPolicyNet, self).__init__()
    self.action_size = K  # Number of arms
    self.hidden_size = d
    self.input_size = 2*d

    self.arm_emb = nn.Embedding(K, self.hidden_size)     # Number of arms
    self.reward_emb = nn.Embedding(2, self.hidden_size)  # For 0 reward or 1 reward
    self.rnn = nn.LSTMCell(input_size=self.input_size,
                          hidden_size=self.hidden_size)
    self.relu = nn.LeakyReLU()
    self.linear = nn.Linear(self.hidden_size, self.action_size)

    self.hprev = None

  def reset(self):
    self.hprev = None

  def forward(self, action, reward):
    arm = self.arm_emb(action)
    rew = self.rew_emb(reward)

    inp = torch.cat((arm, rew), 1)
    h = self.rnn(inp, self.hprev)
    self.hprev = h

    h = self.relu(h[0])
    y = self.linear(h)

    return y
\end{lstlisting}

To train the policy we use the proposed \gradband algorithm as presented in Alg.~\ref{alg:gradband}. We used a batch-size $m=500$ for all experiments.
Along with theoretically motivated steps, we had to apply a few practical tricks:
\begin{itemize}
    \item Instead of SGD, we used adaptive optimizers like Adam or Yogi \citep{zaheer18adaptive}.
    \item We used an exponential decaying learning rate schedule. We start with a learning rate of 0.001 and decay every step by a factor of 0.999.
    \item We used annealing over the probability to play an arm. This encourages exploration in early phase of training. In particular we used temperature = $1/(1-\exp(-5i/L))$, where $i$ is current training iteration and $L$ is the total number of training iterations.
    \item We applied curriculum learning as described in Section~\ref{sec:rnn experiment}.
\end{itemize}
Our training procedure is highlighted in the code snippet below.

\begin{lstlisting}[language=Python, caption=Training overview]
optimizer = torch.optim.Adam(policy.parameters(), lr=0.001)
scheduler = torch.optim.lr_scheduler.ExponentialLR(optimizer, 0.999)

...

probs = rnn_policy_network(previous_action, previous_reward)  
m = Categorical(probs/temperature)   # probability over K arms with temperature
action = m.sample()                  # select one arm
reward = bandit.play(action)         # receive reward

...

loss = -m.log_prob(action) * (cummulative_reward - baseline)  # Eq (3)
loss.backward()   # Eq (9)
optimizer.step()
scheduler.step()

...
\end{lstlisting}